\documentclass[conference]{IEEEtran}

\usepackage[utf8]{inputenc}
\usepackage[english]{babel}
\usepackage{graphicx}
\usepackage{amssymb}
\usepackage{mathtools}
\usepackage{here}
\usepackage{tabularx}
\usepackage{amsmath}
\usepackage{cite}
\usepackage{amsmath,amssymb,amsfonts}
\usepackage{amsthm}
\usepackage{algorithmic}
\usepackage{graphicx}
\usepackage{textcomp}
\usepackage{academicons}
\usepackage{xcolor}
\usepackage{tikz,xcolor,hyperref}
\def\BibTeX{{\rm B\kern-.05em{\sc i\kern-.025em b}\kern-.08em
    T\kern-.1667em\lower.7ex\hbox{E}\kern-.125emX}}

\newtheorem{example}{Example}
\newtheorem{lemma}{Lemma}
\newtheorem{definition}{Definition}
\newtheorem{theorem}{Theorem}
\newtheorem{proposition}{Proposition}

\newcommand{\R}{{\mathbb{R}}}
\newcommand{\Q}{{\mathbb{Q}}}

\definecolor{lime}{HTML}{A6CE39}
\DeclareRobustCommand{\orcidicon}{
	\begin{tikzpicture}
	\draw[lime, fill=lime] (0,0) 
	circle [radius=0.16] 
	node[white] {{\fontfamily{qag}\selectfont \tiny ID}};
	\draw[white, fill=white] (-0.0625,0.095) 
	circle [radius=0.007];
	\end{tikzpicture}
	\hspace{-2mm}
}

\foreach \x in {A, ..., Z}{%
	\expandafter\xdef\csname orcid\x\endcsname{\noexpand\href{https://orcid.org/\csname orcidauthor\x\endcsname}{\noexpand\orcidicon}}
}

\begin{document}

\title{Checking extracted rules in Neural Networks}


\author{\IEEEauthorblockN{Adrian Wurm\orcidA{}}
\IEEEauthorblockA{\textit{Lehrstuhl Theoretische Informatik}} \\
\textit{BTU Cottbus-Senftenberg}\\
Cottbus, Germany \\
wurm@b-tu.de}

%
\maketitle              
\begin{abstract} In this paper we investigate formal verification of extracted rules for Neural Networks under a complexity theoretic point of view. A rule is a global property or a pattern concerning a large portion of the input space of a network. These rules are algorithmically  extracted from networks in an effort to better understand their inner way of working. Here, three problems will be in the focus: Does a given set of rules apply to a given network? Is a given set of rules consistent or do the rules contradict themselves? Is a given set of rules exhaustive in the sense that for every input the output is determined?
Finding algorithms that extract such rules out of networks has been investigated over the last 30 years, however, to the author's current knowledge, no attempt in verification was made until now. A lot of attempts of extracting rules use heuristics involving randomness and over-approximation, so it might be beneficial to know whether knowledge obtained in that way can actually be trusted. 

We investigate the above questions for neural networks with ReLU-activation as well as for Boolean networks, each for several types of rules. We demonstrate how these problems can be reduced to each other and show that most of them are co-NP-complete.
\end{abstract} 
\section{Introduction}

Given the huge success of utilizing Neural Networks, NN for short, in the last decades, such nets are nowadays widely used in all kind of data processing, addressing tasks of varying difficulty.
There is a wide range of applications, the following exemplary references just 
collect non-exclusively some areas for further reading: Image recognition \cite{Krizhevsky}, natural language processing \cite{Hinton}, 
autonomous driving \cite{Grigorescu},
applications in medicine \cite{Litjens}, and prediction of stock markets \cite{Dixon}, just to mention a few. 
Khan et al. \cite{Khan} provide a survey of such applications, a mathematically oriented textbook
concerning structural issues related to Deep Neural Networks is \cite{Calin}.

The huge expressiveness and universality of neural networks comes with the drawback that their inner way of working is barely comprehended at all. Even when given a rather small network, the reasons why it performs its task so well are hardly understandable and the input-output correspondence often can only be evaluated pointwise, and not fathomed in a broader sense. One approach to remedy this problem is rule extraction, that is, reading out macroscopic information from a neural network to obtain knowledge on larger amounts of inputs, which also links this topic to Explainable Artificial Intelligence (EAI), see \cite{eai}. A great number of algorithms performing this task has been introduced over the last three decades, all of them with different advantages and disadvantages, see for example \cite{mekkaoui}, \cite{eai2} or \cite{eai3}. 

A rule is a statement of the form $\varphi(x)\Rightarrow \psi(N(x))$ about a neural network that either holds globally or fails for some input $x$. Here, $\varphi$ refers exclusively to the input and $\psi$ to the output. The optimal set of rules that one could obtain from a net by a rule-extracting algorithm would have the following properties: It is always correct in the sense that whatever it predicts is actually the output of the net, it covers the nets behavior in the sense that for every input at least one rule determines where to it is mapped, it is small meaning it contains only a few rules each of which can be represented in a low amount of space, and the algorithm by which it is obtained is efficient. We will concentrate on the first two properties, that are partly motivated by the fact that among the many different flavors of areas where the use of Neural Networks seems beneficial, some also involve safety-critical systems like autonomous driving 
or power grid management. In such a setting, when security issues become important, aspects of certification come into play \cite{Huang}.

Our first group of main results are as follows: Checking whether such a property holds is a co-NP-complete problem when we restrict us to ReLU or Heaviside as activation. This holds for both real-valued and for Boolean networks. We also show several reductions between the types of rules that we introduce and conclude that the complexities of deciding whether the such rules hold will probably be close together for more complicated activation functions. The second major area of investigation when dealing with extracted rules is deciding whether they are consistent, meaning that they do not contradict each other, and whether for each input they enforce an output. We show that answering these questions is at least as hard as checking whether a rule holds already when using ReLU-activation; as soon as we allow rules with semi-linear descriptions these problems are co-NP-hard. When talking about complexities, we work in the standard Turing model. The paper is organized as follows: In Section \ref{Section:preliminaries} we collect basic notions, recall the definition of feedforward neural nets as well as the precise decision problems that we will discuss.
Section \ref{Section:verex} is about comparing the different types of extracted rules and classifying them by the computational complexity of verifying them. In Section \ref{Section:consistency} we then determine the computational complexity of verifying whether a set of rules is consistent and/or exhaustive. We finish by outlining some future research.

\section{Preliminaries and Definitions}\label{Section:preliminaries}

We start by defining the problems we are interested in. The networks considered are exclusively feedforward. In their most general form, they can process real numbers and contain rational weights. 

\begin{definition} A (feedforward) neural network N is a layered graph that represents a
function of format $\mathbb R^n \rightarrow \mathbb R^m$, for some $n,m \in \mathbb N$.
The first layer with label $\ell = 0$ is called the \emph{input layer} and consists of $n$ nodes called \emph{input nodes}. 
The input value $x_i$ of the $i$-th node is also taken as its output $y_{0i}  := x_i$.
A layer $1 \leq \ell \leq L - 2$ is called \emph{hidden} and consists of $k(\ell)$ nodes called \emph{computation nodes}. The $i$-th node of layer
$\ell$ computes the output $y_{{\ell}i} = \sigma_{{\ell}i} (\sum\limits_j c_{ji}^{({\ell}-1)}y_{({\ell}-1)j} + b_{{\ell}i})$. 
Here, the $\sigma_{{\ell}i}$ are (typically nonlinear)  \emph{activation} functions (to be specified later on) and the
sum runs over all output neurons of the previous layer. The $c^{({\ell}-1)}_{ji}$ are rational constants which are called \emph{weights}, and $b_{{\ell}i}$ is a rational constant  called \emph{bias}.
The outputs of all nodes of layer $\ell$ combined
gives the output $(y_{\ell 0}, . . . , y_{\ell(k-1)})$ of the hidden layer.
The final layer $L - 1$ is called \emph{output layer} and consists of $m$ nodes called \emph{output nodes}.
The $i$-th node computes an output $y_{(L-1)i}$ in the same way as a node in a hidden
layer. The output $(y_{(L-1)0}, . . . , y_{(L-1)(m-1)})$ of the output layer is considered
the output $N(x)$ of the network $N$.

We say that a network is classifying, if the output $N(x)$ is the argmax of the output layer in $\{0,1\}^m$ instead of the vector itself. This means, that an output component is 1 if the corresponding output node of the net returns a value that is at least as big as all the other entries, and 0 otherwise 
\end{definition}

Note that it might happen that a classifying network returns an output where more than one entry is 1, this is, however, very unlikely in practice.

We will deal a lot with logical statements and SAT-reductions, so in the following, when referring to the correctness of a 0-1-output or evaluating a Boolean statement on a numerical input, we interpret 0 as "false" and 1 as "true".

\begin{definition}\label{rules}
Let $N:\R^n\rightarrow\R^m$ be a neural network, $S_\varphi\subseteq\mathbb R^n$ and $S_\psi\subseteq\mathbb R^m$ sets with indicator functions $\varphi$ and $\psi$, and $n,m\in\mathbb N$. A rule is a statement of the form $\varphi\Rightarrow \psi$, that applies to $N$ if for all $x\in\R^n$ we have that $\varphi(x)\Rightarrow \psi(N(x))$. We also denote it by $(\varphi,\psi)$ and call $\varphi$ the conditional part and $\psi$ the conclusion part. A rule is called 

i) propositional, if $\varphi$ and $\psi$ are given as Boolean combinations of paraxial inequalities, meaning intervals of the form $x_i\in[a,b]$, where open intervals and $a,b\in\{\infty,-\infty\}$ are allowed,

ii) oblique, if $\varphi$ and $\psi$ are given as Boolean combinations of half-spaces. Here, the hyperplanes defining the half-spaces do not need to be paraxial.

iii) MofN, if there exist $\square\in\{\leq,=,\geq\}$, $t,r\in\mathbb N$ and $\varphi_i , i\in\{1,...,t\}$ indicator functions on Boolean combinations of half-spaces, so that the conditional part is of the form $\varphi(x)\Leftrightarrow \vert\{ i\in\{1,...,t\} \mid \varphi_i(x)\}\vert \square r$. Intuitively such a rule states that for all $x\in\R^n$, $\psi(x)$ is necessarily the case whenever $x$ evaluates at least $(\square$ is $\geq)$, exactly $(\square$ is $=)$ or at most $(\square$ is $\leq)$  $r$ of the $\varphi_i$ to 1. \footnote{In the literature, $r$ is usually denoted by $m$ and $t$ by $n$, consistent with the term MofN. In order to not get confused with the input dimension $n$ and the output dimension $m$ of the network, our notation differs.}

For all of the above rule types, we consider the corresponding decision problem to find out whether such a rule holds or not.
\end{definition}

\begin{example}
Consider the following network, with all activations being ReLU
\begin{center}
\includegraphics[width=0.45\textwidth]{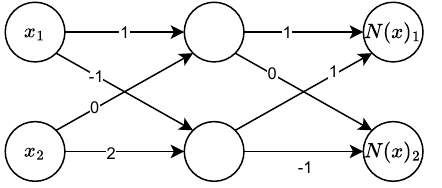}
\end{center}

Possible rules are:

$\varphi_1\Rightarrow\psi_1$, where $\varphi_1(x)=x_1\in(1,2)\lor (x_2\in[2,5]\land x_1\geq3)$ and $\psi_1(N(x))=N(x)_1\geq 1$. this rule is propositional, because $x_1\in(1,2), x_2\in[2,5], x_1\geq3$ and $N(x)_1\geq 1$ are all axial-parallel, it does apply to $N$.

$\varphi_2\Rightarrow\psi_2$, where $\varphi_2(x)=x_1+2x_2=3\land x_1\leq x_2$ and $\psi_2(N(x))=N(x)_1\leq N(x)_2$. It is oblique, because $x_1+2x_2=3, x_1\leq x_2$ and $N(x)_1\leq N(x)_2$ are all linear conditions, but not propositional, because $x_1+2x_2=3$ for example is not paraxial. It does not apply $N$.

$\varphi \Rightarrow\psi_1$, where $\varphi(x)\Leftrightarrow \vert\{ i\in\{1,2\} \mid \varphi_i(x)\}\vert = 1$ and $\varphi_1,\varphi_2,\psi_1$ are as above.
It is MofN and it also applies to $N$. 
\end{example}

\begin{definition}
A neural network $N:\{0,1\}^n\rightarrow\{0,1\}^m$ is called Boolean, if it only uses the Heaviside activation and all weights and biases are in $\{-1,0,1\}$; rules are defined analogously as above.


\end{definition}

Note that a Boolean network is by definition classifying. It is folklore that Boolean circuits can efficiently be translated to Boolean networks computing the same function and vice versa. Also we have that the framework remains as powerful if we demand that $\varphi,\varphi_i,\psi$ to be SAT-formulas. As consequence, it makes no sense to differ between oblique and propositional rules for Boolean networks any more, we therefore restrict ourselves to oblique rules.

In the literature, see for example \cite{fuzzy}, one further type of frequently examined rules are so called fuzzy rules expressed in linguistic concepts such as: "If $x_1$ and $x_3$ are big and $x_2$ is small, then $N(x)_2$ is big". In order obtain a proper decision problem, we interpret them as monotonicity statements, such as "If $x_1\geq y_1$, $x_3\geq y_3$, $x_2\leq y_2$, then $N(y)_2\leq N(x)_2$"
\begin{definition}
Let $N$ be a NN, $A\in\mathbb Q^{\ell\times n}$ and $ i\in\{1,...,k\}$. Then we call $(A,i)$ a monotonicity rule, it applies to $N$ if for all $x,y\in \mathbb R^n$
\[\underset{j}{\operatorname{argmax}}\ N(x)_j=i \land Ax\leq Ay \Rightarrow\underset{j}{\operatorname{argmax}}\ N(y)_j=i.\]
$(A,i)$ is a total monotonicity rule, if
\[\forall x,y\in \mathbb R^n : Ax\leq Ay \Rightarrow N(x)_i\leq  N(y)_i,\]
where $Ax\leq Ay$ is understood componentwise, Boolean total monotonicity rules are defined analogously.
\end{definition}

In contrast to real-valued networks, it is highly likely that a high-dimensional output of a computation performed by a Boolean network does not have a unique argmax. It is therefore not sensible to define Boolean monotonicity rules.

Note that a monotonicity rule $(A,i)$ applies to $N$ if and only if it applies as a total monotonicity rule to the corresponding classifying network.

One general comment concerning the used computational model is in place here: The networks in principle are allowed to compute with real numbers, so the valid inputs belong to some space $\mathbb R^n$. However, the network itself is specified by its discrete structure and rational weights defining the linear combinations computed by its single neurons. Should one move on to more general activations, an issue to be discussed is the computational model in
which one performs the network calculations. If, for example, the sigmoid activation $f(x) = 1/(1 + e^{-x})$ is used, it has to be specified in which sense it is computed by the net, exactly or approximately, and at which costs these operations are being performed. In our case, the networks are assumed to compute exactly, and the decision problems are examined in the Turing model. We presuppose the basic concepts of complexity theory.

\section{Verification of extracted rules}\label{Section:verex}

We now intend to find out how difficult it is to answer the questions defined in Section \ref{Section:preliminaries} and whether their complexities can be related to each other.
\begin{theorem}\label{mainall}
Let $N$ be a network. The following hold for standard as well as for classifying networks: 

a) Propositional rules are special cases of oblique rules, which in turn are special cases of MofN rules.

b) Checking monotonicity rules and total monotonicity rules is linear-time reducible to checking oblique rules. 

c) Checking total monotonicity rules is linear-time reducible to checking monotonicity rules.
 
\end{theorem}
\begin{proof}
a) A propositional rule is a paraxial oblique rule, so obviously a special case. 
An oblique rule in turn is an MofN rule where $\square$ is $=$ and $t=r=1$.

b) Let $(A,i)$ be a monotonicity rule with $A\in\Q^{\ell\times n}$ fitting the dimensions of $N$. We construct a network $N'$ and a propositional rule $(\varphi\Rightarrow\psi)$ so that $(A,i)$ applies to $N$ if and only if $(\varphi\Rightarrow\psi)$ applies to $N'$. The idea is that $N'$ receives both $x$ and $y$ as inputs, computes both $N(x)$ and $N(y)$ in parallel, $\varphi$ demands that $Ax\leq Ay$ and $\psi$ demands that $\underset{j}{\operatorname{argmax}}\ N(x)_j=i \Rightarrow \underset{j}{\operatorname{argmax}}\ N(y)_j=i$. In the case of total monotonicity we demand that $N(x)_i\leq N(y)_i$ instead.

It is easy to see that $Ax\leq Ay$ is semi-linear, so $\varphi$ is given trivially, the same holds for $\psi$ in the case of total monotonicity. For $\psi$, observe that $\underset{j}{\operatorname{argmax}}\ N(x)_j=i \Rightarrow \underset{j}{\operatorname{argmax}}\ N(y)_j=i$ is equivalent to 
\[\bigwedge\limits_{j=1}^m N(y)_j \leq N(y)_i \lor\bigvee\limits_{j=1}^m N(x)_j > N(y)_i,\] which is again semi-linear. Finally, observe that the parallel computation of $N(x)$ and $N(y)$ in $N'$ is just drawing together two copies of $N$, which is again a network in the same model with the same activation functions. 

In the case of classifying networks, replace $\underset{j}{\operatorname{argmax}}\ N(x)_j=i$ with $N(x)_i=1$.

c) We can without loss of generality assume that $N$ has only one output node since total monotonicity only depends on one output dimension, so let $(A,1)$ be a total monotonicity rule for $N$. We want to construct a network $N'$ and an equivalent monotonicity rule $(A',1)$. Total monotonicity is equivalent to $\forall x,y,a: N(x)\geq a\land Ax\leq Ay\Rightarrow N(y)\geq a$. We add another input node and propagate its value interpreted as $a$ to the output layer without interfering with the previous computation, so the new network computes $N':\R^{n+1}\rightarrow \R^2 ,  (x,a)\mapsto (N(x)_1,a)$. To the matrix $A$, we add one zero column to avoid an interference between $x$ and $a$, to the matrix obtained this way we add two rows $(0,0,...,0,1)$ and $(0,0,...,0,-1)$ to enforce that $a$ is constant, so we have 
\[A'=\begin{pmatrix}
A_{11} & \dots & A_{1n}&0\\
\vdots & \ddots & \vdots&\vdots\\
A_{\ell1} & \dots & A_{\ell n}&0\\
0 & 0 & 0 &1\\
0 & 0 & 0 &-1
\end{pmatrix}\in\R^{(\ell+2)\times(n+1)}.\]
Now it holds that $A'(x,a)\leq A'(y,a')$ if and only if $Ax\leq Ay \land a=a'$. This implies that the total monotonicity of $(A,i)$ for $N$ is equivalent to 
\[\forall x,y,a: N(x)>a\land Ax\leq Ay\Rightarrow N(y)\geq a,\] which again is equivalent to $max(N(x),a)=N(x))\land A'(x,a)\leq  A'(y,a')\Rightarrow max(N(y),a)=N(y))$
which in turn is equivalent to $\underset{j}{\operatorname{argmax}}\ N'(x,a)_j=1\land A'(x,a)\leq  A'(y,a')\Rightarrow \underset{j}{\operatorname{argmax}}\ N'(y,a)_j=1$
which finally is equivalent to monotonicity of $(A',1)$ for $N'$.

\end{proof}

\begin{proposition}\label{mainReLU}
Let $N$ be a ReLU-net. 
Deciding whether a set of propositional rules holds for a network is co-NP complete for standard as well as for classifying networks. The same holds for monotonicity rules, oblique rules and MofN rules.

\end{proposition}

\begin{proof} 
By Theorem \ref{mainall} it suffices to show that the verification of propositional rules and the verification of monotonicity rules are co-NP-hard and that the verification of MofN rules is in co-NP.

To show co-NP-hardness, let $S(x_1,...,x_n)$ be an instance of 3SAT. To reduce it to an instance of the verification of propositional rules, we construct a network $N$ with $n$ input nodes by gadgets that evaluates $N(x_1,...,x_n)=1$ whenever $S(x_1,...,x_n)$ holds, and $N(x_1,...,x_n)<1$ otherwise. The first gadget we need is the function $f(x)=\begin{cases}0, \ x\leq0\\
x , \ x\in[0,1]\\
1 , \ x\geq 1
\end{cases}
$, it can be computed by a ReLU-net via $f(x)=ReLU(ReLU(x)-ReLU(x-1))$. For the construction of our problem instance, we set the conditional part $\varphi(x_1,...,x_n)=\chi_{[0,1]^n}(x_1,...,x_n)$ the indicator function on the unit cube. Furthermore, we introduce for each variable $x_i$ hidden nodes $v_i=f(2x_i-1)$ and $\bar v_i=f(1-2x_i)$ corresponding to the literals, so that both are in $[0,1]$, this can be done in two hidden layers. Note that at least one of them is zero, we interpret that as false, and at most one of them is 1, we interpret that as true. In the next two hidden layers, we introduce for every clause $C_i=(x_a\lor x_b\lor x_c)$ in $S$, $1\leq i\leq q$, a node $c_i$ computing $c_i=f(x_a+x_b+x_c)$, where $x_a,x_b,x_c\in\{v_i,\bar v_i \mid 1\leq i\leq n\}$ are the nodes corresponding to the literals. Note that since we are operating in $[0,1]^n$, these nodes can have at most value 1, and this only happens if at least one of them is positive. In the output layer, compute $v=ReLU(\sum\limits_{i=1}^qc_i-(q-1))$. Note that in $[0,1]^n$, this node can have at most value 1, and this only happens if $c_i=1$ are all true. We set $\psi(N(x))\Leftrightarrow N(x)\in[0,1)$, the only way that this will always be the case is when $S$ was not satisfiable. 

To see that this also works as a reduction to monotonicity rules, observe that if $S$ is not feasible, then $N$ is constantly zero. To see this, let $x=(x_1,...,x_n)\in[0,1]^n$ be any input and $y=(y_1,...,y_n)\in\{0,1\}^n$ the componentwise bankers rounding of $x$. Now, due to unsatisfiability, there is a clause $C_i$ in $S$ so that all literals in $C_i$ evaluated in $y$ are zero, and by construction of $v_i$ and $\bar v_i$ the same holds for $x$.

It remains to show that verification of MofN rules is in co-NP. We do this for the case $\square="\geq"$, the cases $\square\in\{=,\leq\}$ are analogous. Let $(\varphi\Rightarrow \psi)$ be an MofN rule, demanding that at least $m$ of $\varphi_1,...,\varphi_n$  are satisfied. Assume that this rule is not satisfied by $x\in\R^n$.

We show that the existence of such an $x$ can be verified non-deterministically. The certificate is not precisely $x$, but consists of the following four information parts: $a\in\{0,1\}^n$ indicating which of the $\varphi_i$ are met in $x$ and which ones are not, the information for every ReLU-node whether it is active (input positive) or inactive (input negative) when evaluated in $x$, and for the input $x$ the information in which of the half-spaces $Ax\leq b$ that are used in the presentation of the $\varphi_i$ it is, as well as the same information for the output $N(x)$ regarding $\psi$.

First, we can check whether the half-space-information actually evaluates the $\varphi_i=1$ and $\psi=0$ in polynomial time as well as whether the postulated amount $m\leq\sum\limits_{i=1}^na_i$ of satisfied formulas is sufficient. Now we translate the remaining information from the certificate into an equation system:

There is an input $x$ that is in all the postulated half spaces $Ax\leq b$ and in none of the forbidden ones $Ax>b$ and for each active node $v=\sum\limits_{i=1}^kv_i \land \sum\limits_{i=1}^kv_i\geq0$ and each inactive node $v=0\land \sum\limits_{i=1}^kv_i\leq0$ we demand that its output is correctly computed as well as that the output is in all the postulated half spaces $AN(x)\leq b$ and in none of the forbidden ones $AN(x)>b$. Observe that all of the conditions are linear. This leaves us with solving a system of linear equations, which is well known to be possible in polynomial time, see \cite{JonssonB}.
\end{proof}

We now want to obtain the corresponding classification results for the discrete case:

\begin{proposition}
Let $N$ be a Boolean network. Then the following problems are co-NP complete:

1.) The verification of Boolean propositional rules.

2.) The verification of Boolean monotonicity rules.

3.) The verification of Boolean MofN rules.

\end{proposition}

\begin{proof}
Membership in co-NP is trivial for all problems, a suitable certificate is always an input vector in $\{0,1\}^n$ that leads to failure. It remains to show completeness:

1.) To reduce an instance $\varphi$ of SAT on the complement of the Boolean verification of propositional rules, set $\varphi$ as the conditional part for any network $N$ and a formula that is trivially wrong as output class.

2.) For completeness, we reduce a SAT-formula $F(x_1,...,x_n)$ on $n$ variables to the complement of deciding monotonicity rules. First, check whether $F(1,...,1)$, in that case return true, otherwise return the instance $(N_F,A,1)$, where $N_F:\{0,1\}^n\rightarrow \{0,1\}$ evaluates an input $(x_1,...,x_n)$ to 1 if and only if $F(x_1,...,x_n)$ holds, and $A\in\{0,1\}^{n\times n}$ is the identity matrix.

3.) As for real-valued networks, Boolean propositional rules are special cases of MofN rules with $t=r=1$.
\end{proof}

\section{Consistency and Exhaustiveness}\label{Section:consistency}
One property that we want a set of rules to have is that it should holistically depict the network's behaviour. In this section, we do not examine any more whether a single rule is correct or not, but start focusing on larger sets of rules. Here, two questions are in focus: Is a set of rules actually obeyable, meaning is there a network that obeys all of them at the same time? And are these rules so restrictive that one can make choices when designing a network according to them, or is there just one unique function that obeys all the rules? This is referred to as "fidelity" or "completeness" in literature, see for example \cite{fidel1}.  We renamed those properties, because in the literature they often concern the accordance with training data rather than with the network.
In this framework, we do not work with a given network any more, but instead start with a set of rules and search for possible networks obeying them, or we are given a set of rules and ask ourselves whether there is more than one network obeying all of them.
\begin{definition} 
A set of rules is called \emph{consistent}, if there exists a network that obeys all of these rules, otherwise it is called contradictory. 

A set of rules is called \emph{exhaustive}, if for every network-input $x\in\R^n$, there exists at most one output $y\in\R^m$ so that $N(x)=y$ is in accordance with all rules, otherwise it is called defective.

The obvious decision problems consistency and exhaustiveness are defined for all of the rule types given in Section \ref{Section:preliminaries}.
\end{definition}

An extracting algorithm obviously made a mistake if the extracted rules are not consistent, consistency is therefore a desirable property. The same holds for exhaustiveness, because we want the set of rules to provide as much knowledge on the function as possible and to be as close as possible to "replacing" the network in the sense that it already contains all information. 
\begin{example}
Let $S=\{s_1,s_2\}$ be a set of two rules.

The optimal constellation is that $S$ is consistent as well as exhaustive. This is, for example, the case for the propositional rules $s_1=(\varphi_1\Rightarrow\psi_1)$ with $\varphi_1(x)=x\leq0, \psi_1(N(x))=(N(x)=0)$ and $s_2=(\varphi_2\Rightarrow\psi_2)$ with $\varphi_2(x)=x\geq0, \psi_2(N(x))=(N(x)=0)$, which allow precisely the constant zero function. Had we dropped $s_2$, the rule $s_1$ would allow both the constant zero function and the ReLU-function, it is therefore consistent but not exhaustive.
Note that a set of rules can be both defective and contradictory, an example is $s_1=(\varphi_1\Rightarrow\psi_1)$ with $\varphi_1(x)=x\leq0, \psi_1(N(x))=(N(x)=0)$ together with $s_2=(\varphi_2\Rightarrow\psi_2)$ with $\varphi_2(x)=x\in[-1,1], \psi_2(N(x))=(N(x)=1)$ because $N(0)$ cannot be 0 and 1 at the same time, but also $N(2)$ is not restricted at all. Also a set of exhaustive rules is not necessarily consistent, this is for example the case had we replaced $\varphi_2(x)=x\geq-1, \psi_2(N(x))=(N(x)=1)$ in the previous constellation.
\end{example}

\begin{lemma}[\hspace{1sp}\cite{arora}]\label{universal}
Every continuous piecewise-linear function $f$ with rational coefficients can be realized by a ReLU-network $N$, i.e., $f(x)=N(x)$ for all inputs $x$. 

\end{lemma}
It is an easy consequence, that for any finite set of data points $(x_i,y_i)\in\Q^{n+m}$ with pairwise different inputs $i\neq j\Rightarrow x_i\neq x_j$, there exists a ReLU-network $N$ with $N(x_i)=y_i$ for all $i$.

\begin{theorem}
Let $S=\{s_1,...,s_n\}$ be a set of propositional, MofN and oblique rules so that the sets defined by their conditional parts are separated by closed neighbourhoods and no conclusion part is the constant zero function. Then there exists a ReLU network obeying all the rules.
\end{theorem}

\begin{proof}
By Lemma \ref{universal}, it suffices to find a continuous semi-linear function obeying the rules. For each rule, choose a value that the conclusion part allows and tie that the function has that precise value all over the corresponding conditional part. Separability of a finite number of semi-linear sets leads to the existence of a constant $c>0$ so that $\varphi_i(x)\land \varphi_j(y)$ with $i\neq j$ enforces $d(x,y)\geq c$. The function can therefore be extended to a semilinear and continuous function on $\R^n$.
\end{proof}

\begin{theorem}\label{conex}
Let $R$ be a set of rules. Then:

1.) Deciding consistency and exhaustiveness of monotonicity rules and total monotonicity rules is in P if we only allow networks to work with ReLU-activation.

2.) Deciding consistency of propositional rules is co-NP-hard if we only allow networks to work with ReLU-activation. This also holds for obliques and MofN rules.

3.) Deciding exhaustiveness of propositional rules is co-NP-hard. This also holds for obliques and MofN rules.
\end{theorem}

\begin{proof}
1.) Every constant function obeys every monotonicity rule, so consistency is always and exhaustiveness never the case.

2.) To reduce an instance $S(x_1,...,x_n)$ of 3SAT to consistency of propositional rules, it suffices to construct a set $R$ of rules on $[0,1]^n$ that are consistent if and only if $S$ is not satisfiable. The set $R$ contains a rule stating that the network $N$ should compute the constant zero function along with rules that demand that $N(x_1,...,x_n)=S(x_1,...,x_n)$ for all $x_1,...,x_n\in\{0,1\}$. 

Such rules exist because the set $S$ interpreted as subsets of $\{0,1\}^n$ can be described by axial-parallel conditions in polynomial size in the representation of $S$. If $S$ is not satisfiable, then the network computing $N\equiv 0$ obeys all the rules, if not then no such network exists and the rules are not consistent.

This ongoing, as in 3.), obviously works the same for Boolean networks. By the inheritances shown in Theorem \ref{mainall}, this also holds for obliques and MofN rules, as in 3.).

3.) Again we reduce an instance $S(x_1,...,x_n)$ of 3SAT, this time $R$ has to be non-exhaustive if and only if $S$ is satisfiable. The set $R$ contains a rule stating that whenever $S(x_1,...,x_n)=0$, the network $N$ should compute zero whenever all over $\bigtimes_{i=1}^n[x_i,x_i+1]$. This means that the constant zero function is always possible, and as soon as $S(x_1,...,x_n)=1$ at some point, the function has the choice either to be 0 or 1.

Such rules exist for the same reason as in 2.)
\end{proof}

The following result demonstrates the limits of the framework of extracted rules:
\begin{theorem}\label{unique}
Let $R$ be a finite consistent and exhaustive set of rules for a ReLU network assume $N$ obeys all rules in $R$.

1.) If $R$ contains only propositional, MofN and oblique rules, then $N$ can only compute a constant function.  

2.) $R$ cannot solely consist of monotonicity and total monotonicity rules. 
\end{theorem}

\begin{proof}
1.) We prove the statement by contradiction, so assume there were a set of rules $R=\{(\varphi_1\Rightarrow\psi_1),(\varphi_2\Rightarrow\psi_2),...,(\varphi_k\Rightarrow\psi_k)\}$ and a ReLU-network $N$ so that $N$ does not compute a constant function, $N$ obeys all the rules in $R$ and no function other than $N$ that can be computed by a ReLU-network obeys all the rules in $R$. Note that in the above, by abuse of notation $\varphi_i$ is always the rules conditional part concerning the input. In the case of MofN-rules, for example, it does not coincide with one of the $\varphi_1,...,\varphi_t$ from Definition \ref{rules}. We now want to alter $N$ in such a way that the function computed by the new network $N'$ differs, but only at an extend so small that all rules in $R$ are still obeyed, this would obviously contradict exhaustiveness.

First, note that for a non-constant function $N$ there must necessarily exist some polytope $P\subseteq \R^n$ in the input space containing a full-dimensional box on which $N$ has a constant non-zero slope. The reason is the following: Every sign-pattern $a\in\{0,1\}^\ell$ determining which nodes are active and which ones are inactive, where $\ell$ is the overall amount of nodes, defines a subset of inputs $X_a=\{x\in\R^n \mid \text{pattern of $x$ is } a\}$, restricted on which $N$ is linear. By continuity, the union $\bigcup\limits_{y\in A}\{x\in\R^n \mid N(x)=y\}$ for a finite set $A$ cannot have full measure, so there must exist a set $X_a$ with non-zero measure on which $N$ has non-zero slope. 

Next, we argue that $P$ can be chosen in a way so that for each rule $(\varphi_i\Rightarrow\psi_i)\in R$, the premise $\varphi_i$ either holds everywhere or nowhere in $P$. To see this, define for each $\alpha\in\{0,1\}^k$ the set $P_\alpha=\{x\in P \mid \bigwedge\limits_{\alpha_i=1} \varphi_{i}(x)\}$. Note that $P$ is the finite union of these semi-linear sets $P_\alpha$, so at least one of them must have non-zero measure, take it as the new $P$. Without loss of generality, $P$ is the unit cube $[0,1]^n$. This can be assumed because every semi-linear set with non-zero measue contains a hypercube with equal side lenghths, that can be mapped to the unit cube by a linear function, which can obviously be performed in a ReLU-network. Further, assume that the partial derivative in direction $x_1$ is not zero, otherwise permutate the variables. All conclusions $\psi$ are local, meaning they only ever consider an output $N(x)$ and no slope or comparison with another point $N(y)$. The idea is therefore that it suffices if the new function/network $N'$ has the following properties:
\begin{description}
\item[1.)] For all $x\notin[0,1]^n$ we have $N'(x)=N(x)$.
\item[2.)] For all $x\in[0,1]^n$ we have $N'(x)=N(y)$ for some $y\in[0,1]^n$.
\end{description}

We claim that these rules are obeyed, if the function on $x\in[0,1]^n$ is computed as follows:

Find the point $a(x)$ on the boundary of $[0,1]^n$ and the parameter $\lambda(x)\in[0,1]$ so that $x=\lambda(x) a(x) + (1-\lambda(x))([\frac12]^{n})$ is a convex combination of $a(x)$ and the center of the cube $[\frac12]^{n}$, where $[x]^{n}:=(x,x,...,x)\in\mathbb R^n$ denotes the vector consisting only of entries $x$. 
We then choose $N'(x)=N(\lambda(x) a(x) + (1-\lambda(x))(\frac34,[\frac12]^{n-1}))$ the function value of $N$ at the same convex combination but where the center of the cube is moved in direction $x_1$ by $\frac14$. Note that this function $N'$ does not coincide with $N$, because $N'([\frac12]^{n})=N(\frac34,[\frac12]^{n-1})\neq N([\frac12]^{n})$ for the slope in direction $x_1$ was assumed to be non-zero. It remains to show that the input manipulation described by the convex combinations can be computed in a ReLU-network.

To see this, observe that the input manipulation is equivalent to moving a point in direction $x_1$ by half of its minimal distance to any boundary, including the boundary faces $x_1=0$ and $x_1=1$. These distances inside the unit cube are given by $ReLU(x_i)$ for the boundary $x_i=0$ and  $ReLU(1-x_i)$ for the boundary $x_i=1$, note that they are zero on the outside of the cube of the respective boundary. A minimum can always be computed by a ReLU-network via 
\[min(a,b)=\frac{a+b}{2}-ReLU(\frac{a-b}{2})-ReLU(\frac{b-a}{2})\]
The desired input manipulation \[f(x)=\begin{cases}
x &x\notin[0,1]^n\\
x+min\{\vert x_i\vert,\vert x_i-1\vert \mid i\in\{1,...,n\}\}&else
\end{cases}\]
can therefore be computed and $N'=N\circ f$ is different from $N$ but still obeys all the rules.

2.) Follows by Theorem \ref{conex}, 1.). 
\end{proof}

\section{Conclusion and Further Questions}
We classified several verification tasks for extracted rules and proposed a framework in which they can be assessed. Our overall insight is, that these verification tasks are in co-NP, many of them complete for co-NP, as long as the describing data of the rules and the net is semi-linear, including that the activation function used in the net should be ReLU. Some remaining open questions are

1.) Given a consistent set of MofN rules, can the size of the smallest witnessing network be bounded in the size/amount of rules? Is it polynomial? Does this imply that verifying consistency of MofN rules is in $\Sigma_2^P$? If so, is it complete for $\Sigma_2^P$?

2.) Is it possible to efficiently transform linear conditions into axial-parallel ones? If yes, can this be done by a (semi-)linear transformation that can be performed by additional network layers? Under what conditions would this imply that oblique rules can be reduced to propositional rules?

3.) Does Theorem \ref{unique}, 1.) also hold if additionally monotonicity and total monotonicity rules are allowed? Or can a function be found that is uniquely determined by a set of such rules?

\textbf{Acknowledgment}: I want to thank Klaus Meer for helpful discussions.

\bibliography{literatur}

\bibliographystyle{plain}

\end{document}